\documentclass{article}

\PassOptionsToPackage{numbers,sort&compress}{natbib}
\usepackage[preprint]{neurips_2026}

\usepackage[utf8]{inputenc} % allow utf-8 input
\usepackage[T1]{fontenc}    % use 8-bit T1 fonts
\usepackage{hyperref}       % hyperlinks
\usepackage{url}            % simple URL typesetting
\usepackage{booktabs}       % professional-quality tables
\usepackage{amsfonts}       % blackboard math symbols
\usepackage{nicefrac}       % compact symbols for 1/2, etc.
\usepackage{microtype}      % microtypography
\usepackage{xcolor}         % colors

\title{\ourmethod: Fast Drug Discovery by Learning from Sparse Data}

\author{%
  Rahul Nandakumar \\
  McCombs School of Business\\
  University of Texas at Austin\\
  Austin, TX \\
  \texttt{rahul.nandakumar@utexas.edu} \\
  \And
    Ben Fauber \\
  NVIDIA\\
  \texttt{bfauber@nvidia.com} \\
  \And
  Deepayan Chakrabarti \\
  McCombs School of Business\\
  University of Texas at Austin\\
  Austin, TX \\
  \texttt{deepay@utexas.edu} \\
}

\usepackage{bm, xspace, dsfont}
\usepackage{algorithm, algpseudocode}
\usepackage{xparse}
\usepackage{nicematrix}
\usepackage{amsmath, amsthm}
\usepackage{subcaption}
\usepackage{caption}
\usepackage{xcolor}
\usepackage{wrapfig}

\usepackage{multirow}
\usepackage{booktabs}
\usepackage{float}

\usepackage{enumitem}
\usepackage{soul}

\usepackage{stfloats}

\sethlcolor{red!30}

\makeatletter
\NewDocumentCommand{\LeftComment}{s m}{%
  \Statex \IfBooleanF{#1}{\hspace*{\ALG@thistlm}}\(\triangleright\) #2}
\makeatother

\DeclareMathOperator*{\argmin}{arg\,min}

\newcommand{\txtred}[1]{\color{red}#1\color{black}\xspace}

\newcommand{\ourmethod}{\textsc{SPADE}\xspace}
\newcommand{\ourmethodwithreason}{\textsc{SPADE} ({\bf S}parse-data {\bf P}redictions for {\bf A}ccelerating {\bf D}rug {\bf E}xploration)\xspace}

\newcommand{\bx}{{\ensuremath{\bm x}}\xspace}
\newcommand{\bv}{{\ensuremath{\bm v}}\xspace}
\newcommand{\bw}{{\ensuremath{\bm w}}\xspace}

\theoremstyle{plain}
\newtheorem{theorem}{Theorem}[section]

\theoremstyle{definition}

\theoremstyle{remark}
\newtheorem{remark}[theorem]{Remark}
\begin{document}

\maketitle

\begin{abstract}
% Many diseases stem from proteins that are out of balance.
Drug discovery seeks molecules (ligands) that bind strongly and selectively to a target protein.
However, fewer than $5\%$ of candidate ligands pass the bar for even the early stages of drug discovery.
Furthermore, we want methods that work for novel proteins for which we have no prior data.
Starting from scratch, we have to iteratively select and test candidate ligands such that we find enough ligands of the desired quality in as few tests as possible.
Our proposed algorithm, named \ourmethod, introduces a novel approach to ligand selection that requires only $40$ tests on average to find $10$ high-quality ligands.
In one-vs-one comparisons, SPADE outperforms deep learning and Bayesian optimization methods on more proteins, achieving median improvements of $7\%-32\%$ in sample efficiency.
%In direct comparisons, \ourmethod requires $7\%-32\%$ fewer tests than state-of-the-art deep learning and Bayesian optimization methods, and 
\ourmethod is also 10x faster than its closest competitor at scoring candidate drugs. Dataset and code is available at \url{https://anonymous.4open.science/r/SPADE_Fast_Drug_Discovery_by_Learning_from_Sparse_Data-F028/README.md}
% Finally, to support future research, we constructed and will release a new dataset of $1.5M$ 
% %\textcolor{red}{(3.5 M instead?)} 
% ligand-protein affinities from PubChem that complements existing datasets.
%{\em Our code and new benchmark dataset, containing 3.5 million protein-ligand affinities, will be released after incorporating comments from reviewers.}
\end{abstract}

\section{Introduction}
\label{sec:intro}

Proteins regulate nearly all biological processes in the human body, and disruptions to their activity can lead to disease. A central goal of drug discovery is therefore to design small molecules (ligands) that bind strongly and selectively to a target protein \cite{Knowles2003NRDD, Waring2015AnAO}. Despite decades of progress, this process remains highly inefficient: fewer than 5\% of candidate ligands succeed even in early-stage screening. The challenge is particularly acute for \emph{novel} proteins, where little or no prior data is available.

In early-stage drug discovery, researchers iteratively run design-make-test-analyze (DMTA) cycles \cite{PLOWRIGHT201256}. In each cycle, a small number of ligands are selected, synthesized, and experimentally evaluated. The outcome of interest is the ligand’s binding affinity, typically measured via the pIC50 (PIC) value \cite{Swinney2011HowWN}, where higher values indicate stronger binding. Since experimental tests are slow and expensive, the key objective is to identify a small number of (say, $5$) high-quality ligands (e.g., PIC $\geq 8$) using as few tests as possible. We refer to this objective as a \emph{race-to-$k$} problem, that is:

\begin{center}
    find $k$ ligands with PIC above a desired threshold in as few DMTA cycles as possible.
\end{center}
A PIC of $8$ or $8.5$ is often a favorable starting point for the later stages of drug discovery~\cite{LombardinoLoweNRDD2004}.
Hence, we focus on the {\em race-to-8} and {\em race-to-8.5} problems.

This setting presents several fundamental challenges. First, the search space is vast, consisting of millions of candidate ligands represented by high-dimensional embeddings. Second, the signal is extremely sparse: only a small fraction of ligands meet the desired affinity threshold (e.g., only 7\% of candidates have PIC $\geq 8$, and just 2.7\% have PIC $\geq 8.5$).
Third, we start with no labeled data and must learn entirely from sequential, adaptively collected observations. These properties make standard machine learning approaches difficult to apply effectively.

Recent work approaches this problem using active learning and Bayesian optimization. These methods estimate the PIC of each candidate ligand and combine it with an uncertainty estimate (e.g., from a Gaussian process) to guide selection~\cite{griffiths2024gauche}. However, in our setting, accurate estimation is difficult due to extreme data sparsity and high dimensionality: we begin with no data, and each ligand is represented by a $2{,}048$-dimensional embedding. Moreover, most ligands fail to meet the desired threshold.
%(only 7\% exceed PIC 8, and just 3\% exceed 8.5). 
As a result, estimating PICs is both difficult and unnecessary for the race-to-$k$ objective.

Another line of work constructs protein embeddings~\cite{rives2019biological,passaro2025boltz2}. However, directly predicting affinity using such embeddings achieves limited accuracy~\cite{gorantla24proteins,FauberLPI2024}. Since we focus on novel proteins, we assume that no similar proteins are available, and thus do not rely on protein embeddings. If such information is available, it can be incorporated as a preprocessing step to filter the candidate ligand set (we show such an experiment in Appendix~\ref{sec:app:expdetails}).

Finally, we emphasize that our work focuses on {\bf ligand affinity, which is fundamentally distinct from molecular docking}~\cite{stein_property-unmatched_2021}.
Our approach is binding-mode agnostic and focuses on continuous affinity (PIC) values. 
In contrast, datasets such as DUD-E \cite{mysinger2012directory} provide only docking scores, while LIT-PCBA \cite{tran2020lit} offers binary active/inactive labels. 
Neither dataset captures the continuous affinity information essential for ranking and prioritizing ligands, rendering them unsuitable for our objective.

% Finally, we emphasize that {\bf ligand affinity prediction (our problem) is different from molecular docking}~\cite{stein_property-unmatched_2021}. 
% Our approach is binding-mode agnostic. 
% Datasets such as DUD-E \cite{mysinger2012directory} provide docking scores rather than PIC values.
% LIT-PCBA \cite{tran2020lit} provides only binary labels, but drug discovery needs continuous values to prioritize ligands. 
% Hence, these datasets are unsuitable for our continuous, affinity-based objective.

\textbf{Our contributions:}
\begin{itemize}
    \item \textbf{Problem formulation:} We cast early-stage drug discovery for novel proteins as a sparse, sequential \emph{race-to-$k$} problem, where the objective is to identify $k$ high-affinity ligands using as few DMTA cycles as possible.

    \item \textbf{Classification-based selection (Sec.~3):} We propose \ourmethodwithreason, which replaces full PIC estimation with the simpler task of predicting whether a ligand can outperform the current $k^{th}$-best ligand.
    Thus, \ourmethod focuses directly on improving the top-$k$ set.

    \item \textbf{Robust learning under extreme sparsity (Sec.~3):} Since the positive class contains only $k\!-\!1$ ligands, standard classifiers struggle. Instead, we introduce a robust classifier that minimizes the expected loss over a Gaussian centered at each positive example, reducing overfitting to noise. We derive a closed-form expression for the expected loss to avoid sampling.
    We then combine these classifiers to efficiently identify the best candidate ligands.

    \item \textbf{Large-scale dataset (Sec.~4):} We built a new $1.5$M-entry PubChem-derived dataset, complementary to existing datasets, to test sequential ligand discovery under realistic sparsity.

    \item \textbf{Empirical evaluation (Sec.~4):} Across 100 proteins, \ourmethod\ consistently outperforms state-of-the-art baselines, requiring $7\%\!-\!32\%$ fewer ligand tests to reach target PICs and achieving a $10\times$ speedup in scoring ligands.
\end{itemize}

\section{Related Work}
\label{sec:related}

% LPI affinity prediction
%   - ML and Deep Learning models
%   - They often use small binary datasets or small search spaces of 48 proteins, 48 ligands (Naik)
%   - Less useful for ranking the top ligands for screening
%   - FEP models
% Data representations
%   - Embeddings, graphs-based models
%   - Have limited impact on prediction accuracy.`
% Active learning
%   - Goal is to learn conditional dist over entire feature space, which may not be optimal for minimizing DMTA cycles
%   - Effort into creating balanced training set, since good ligands are very rare.
%   - Kangas finds exploitation is best for our goal. Reker finds that exploration may not be as rapid as random sampling in terms of target coverage speed.
%   - We focus on starting from scratch; race-to-8

We review prior work on Ligand-Protein Interaction affinity (LPI) prediction, protein and ligand embeddings, and active-learning for drug discovery.

\textbf{LPI prediction:}
Several papers cast LPI prediction as a binary classification task.
Then, they apply machine learning models to learn affinities
\cite{Oliveira2024InferringMI, Kimber2021DeepLI, MartinAllAssayMax2QSAR2019, Mayr2018LargescaleCO, Martin2011ProfileQSARAN, Yamanishi2008PredictionOD, FaulonMisraDTIPred2007}.
Deep learning models have also been developed for this task
\cite{Huang2020DeepPurposeAD, Huang2020MolTransMI, Li2020MONNAM, ztrk2019WideDTAPO, Whitehead2019ImputationOA, Lee2018DeepConvDTIPO, beyondthehypeDL2017, Wen2017DeepLearningBasedDI}.
However, simple binary labels, such as active/inactive, oversimplifies the continuous nature of binding affinities.
Hence, these methods have been less useful for ranking ligands for drug screening
\cite{Sadybekov2021SynthonbasedLD}.
In contrast, \ourmethod considers only the current top-$k$ ligands as the positive class, so the labels shift over the DMTA cycles. 

Another approach is to use physics-based methods like free-energy perturbation.
These methods attempt to mimic the protein-ligand binding interactions via computational simulations. 
They can offer more precise results but are computationally intensive
\cite{Wang2015AccurateAR, Ross2023TheMA, doi:10.1021/acs.jcim.0c00900}.
Hence, even if these methods are used instead of lab tests in DMTA cycles, we still need to minimize the number of cycles.

\textbf{Protein and ligand embeddings:}
Existing approaches try to predict ligand-protein interaction affinity using vector embeddings \cite{Kimber2021DeepLI, Kalakoti2022TransDTITL} and graph-based models \cite{Svensson2024HyperPCMRT, Chatterjee2023ImprovingTG, Thafar2022Affinity2VecDB, Kimber2021DeepLI} for protein and ligand embeddings.
Several specialized embeddings also exist for proteins~\cite{passaro2025boltz2, rives2019biological} and for ligands~\cite{rogers_extended-connectivity_2010, durant_reoptimization_2002}.
However, embedding methods have limited impact on prediction accuracy \cite{gorantla24proteins, FauberLPI2024}.

We focus on drug discovery for novel proteins.
So, we assume there is no side information available about similar proteins.
In our setting, we only have one protein (the target protein), so there is no need for protein embeddings.
Our work is agnostic to the choice of ligand embeddings, and we show results using the popular ECFP~\cite{rogers_extended-connectivity_2010}, MACCS~\cite{durant_reoptimization_2002}, {and ChemBERTa~\cite{chithrananda2020chemberta} ligand embeddings.

\textbf{Active learning and Bayesian optimization:}
Since most ligands have poor LPI, existing datasets are imbalanced.
Active learning has been used to develop balanced training datasets for LPI prediction via explore-exploit strategies.
However, the best way to find good interacting pairs is to only exploit
\cite{kangas_efficient_2014}.
Also, active learning is often used to learn the entire protein-ligand interaction landscape.
In contrast, the goal of early-stage drug discovery is to quickly find a few high-quality ligands.
Learning about the entire landscape is not necessary.

Bayesian optimization is another approach for selecting the ligands to test in each DMTA cycle. 
Here, we rank candidate ligands by combining their estimated PIC with uncertainty scores obtained via Gaussian processes
~\cite{griffiths2024gauche}.
The top-ranked ligands are tested in the next DMTA cycle, and all estimates are recomputed using the new data.
But in our setting, there are very few ligands with known PICs, and the ligand embedding is high-dimensional.
This leads to large uncertainty for ligand PIC estimates.
\ourmethod avoids estimation and instead uses robust classification.
We show that we outperform Bayesian optimization methods, and are also 10x faster than the closest competitor.

\textbf{Structure-based and deep virtual screening:}
These methods analyze the ligand and protein structures to predict docking.
However, docking is different from affinity prediction.
Also, the best-performing methods typically require seconds of
compute per ligand even with GPU acceleration~\cite{passaro2025boltz2}, while SPADE is 3-4 orders of magnitude faster.
Finally, such methods do not learn from DMTA cycles, unlike \ourmethod.
We note that screening methods can filter the set of ligands as a preprocessing step for \ourmethod.
This can improve accuracy by 18\%-31\% (Appendix~\ref{sec:app:expdetails}).

\section{Proposed Method}
\label{sec:prop}

Our goal is to identify several high-affinity ligands using as few DMTA cycles as possible. This setting presents several key challenges:

\noindent\textbf{(1) Rare targets in a vast search space.}
We must search among $10^5$--$10^6$ candidate ligands, with a median PIC of $5.9$. The distribution is highly skewed: about $75\%$ of ligands have PIC $< 7$, only $7\%$ have PIC $\geq 8$, and just $2.7\%$ exceed $8.5$.

\noindent\textbf{(2) Sparse data.}
For novel proteins, we begin with no labeled data. Each data point must be obtained via costly and time-consuming ligand tests, so the model must learn effectively from little data.
Furthermore, since our goal is to improve the top-$k$ ligands in each DMTA cycle, the ``positive class'' remains extremely small (namely, the $k-1$ ligands that are currently the best).

\noindent\textbf{(3) High dimensionality.}
Each ligand is represented by a $2{,}048$-dimensional ECFP vector~\cite{rogers_extended-connectivity_2010}. This high dimensionality, combined with data sparsity, makes accurate PIC estimation difficult.

\noindent\textbf{(4) High throughput requirements.}
In each DMTA cycle, we must score $10^5-10^6$ candidate ligands.
Hence, the scoring needs to be fast and efficient.

%Traditional approaches, designed for large i.i.d.\ datasets and fixed prediction targets, are ill-suited to this setting.

% \noindent\textbf{(4) Moving target.}
% Our goal is to find $k$ ligands with high PICs. As more ligands are tested, the $k^{\text{th}}$-best PIC increases. Thus, the decision boundary evolves over time, while the positive class remains extremely small (only $k$ points).

% \noindent\textbf{(4) Correlated data.}
% Ligands are selected adaptively based on previous outcomes, leading to correlated samples. This violates the i.i.d.\ assumptions underlying many standard learning methods.

\subsection{Overview of \ourmethod}

\ourmethod departs from standard approaches in two fundamental ways: it replaces global affinity estimation with a targeted classification objective, and it incorporates robustness directly into the learning process. Together, these design choices yield a method that is well-suited to extreme data sparsity, requires minimal tuning, and is significantly faster at scoring candidate ligands.

\paragraph{Classification instead of estimation.}
Most existing methods attempt to estimate the PIC for every candidate ligand.
%, effectively learning the full affinity landscape. 
In our low-data setting, this is both difficult and unnecessary. Instead, \ourmethod focuses on a simpler and more targeted task: predicting whether a ligand can outperform the current $k^{\text{th}}$-best ligand. 
%This avoids modeling the full landscape and 
This directly optimizes for our goal of improving the top-$k$ set in each DMTA cycle.
%, which is the objective in each DMTA cycle.

\paragraph{Robustness under extreme sparsity.}
The positive class consists of only the top-$k$ ligands seen so far.
%, making standard empirical risk minimization highly sensitive to noise. To address this, 
To be robust against noise, \ourmethod optimizes an expected loss over a distribution centered at each positive ligand. This acts as a principled regularizer, encouraging the model to learn stable patterns rather than overfitting to a handful of points. We further design the loss so that this expectation admits a closed-form expression, avoiding the need for sampling.

%\medskip
In each DMTA cycle, \ourmethod trains one robust classifier for each of the current top-$k$ ligands. It then scores all untested ligands using a weighted combination of these classifiers, and selects the highest-scoring candidates for testing in the next cycle. This process iteratively refines the top-$k$ set.
Next, we discuss these steps in detail.

\subsection{Classifier for a Single Top Ligand}
After every DMTA cycle, we have a dataset $\mathcal{D}_{seen}$ of the ligands tested so far, along with their PICs.
Let $\mathcal{S}^+\subset \mathcal{D}_{seen}$ be the set of the best few ligands seen so far, and let $\mathcal{S}^-:=\mathcal{D}_{seen}\setminus \mathcal{S}^+$.
For each ligand $i\in\mathcal{S}^+$, we build a classifier $\mathcal{C}_i$ that separates $i$ from all the ligands in $\mathcal{S}^-$:
\begin{equation}
\begin{aligned}
    \mathcal{C}_i := \argmin_{\mathcal{C}\in\mathcal{F}}  \left(
\frac{\sum_{j\in\mathcal{S}^-} \ell(\mathcal{C}(\bx_j), y=-1)}{|\mathcal{S}^-|}
+ E_{\bx\sim\mathcal{N}(\bx_i, \sigma^2 I)} \left[\ell(\mathcal{C}(\bx), y=1)\right]
\right)
\end{aligned}
\label{eq:loss}
\end{equation}

where $\bx_j\in\mathbb{R}^d$ denotes the embedding of ligand~$j$, $\mathcal{F}$ is the search space over classifiers, $\mathcal{C}$ is any classifier in $\mathcal{F}$, $C(\bx)$ is the predicted score of $C$ on a ligand with embedding $\bx$, and $\ell(s, y)$ is the loss if a score $s$ is predicted for a ligand with class $y\in\{+1, -1\}$.
%$y=-1$ (ligand in $\mathcal{S}^-$) or $y=1$ (ligand in $\mathcal{S}^+$).

{\textbf{Difference from empirical risk minimization:}
Instead of the empirical loss on $\mathcal{S}^+$, we use the {\em expected loss} over a Gaussian distribution. This distribution is centered at $\bx_i$ for each $i\in\mathcal{S}^+$.
Intuitively, $\bx_i$ is a sample from a distribution of similar-PIC ligands.
Since we do not have enough data to reconstruct that distribution, we use a Gaussian distribution as an approximation.
The width $\sigma$ of the Gaussian reflects our uncertainty about the distribution.
Larger $\sigma$ means greater uncertainty and a more robust classifier $\mathcal{C}_i$.
A well-chosen $\sigma>0$ ensures that $\mathcal{C}_i$ identifies the specific features of $\bx_i$ that (a)~set it apart from the low-PIC ligands in $\mathcal{S}^-$ and (b)~are not due to randomness.

{\textbf{Choice of loss function:}
We choose $\ell(\cdot, \cdot)$ so that  Equation~\ref{eq:loss} has a closed-form formula.
Specifically, let any classifier $\mathcal{C}\in\mathcal{F}$ be parametrized by the pair $(c, \bw)\in\mathbb{R}\times\mathbb{R}^d$.
The score for $\mathcal{C}$ on a ligand $\bx\in\mathbb{R}^d$ is given by $\mathcal{C}(\bx):=c+\bw^T\bx$.
We use the loss $\ell(\mathcal{C}(\bx), y):=\max\{0, 1-y\cdot \mathcal{C}(\bx)\}$ for $y\in\{+1, -1\}$.
These choices offer two main benefits.
First, we can learn from all features without making the model too complex and unstable.
This is important for robustness, as the data is both small and imbalanced in our problem.
Also, we can calculate the expected-loss term in Equation~\ref{eq:loss} in closed form, {\em without the need for sampling}.

% We use the hinge loss $\ell(\mathcal{C}(\bx), y):=\max\{0, 1-y(c+\bw^T\bx)\}$ for $y\in\{+1, -1\}$.

\begin{theorem}
Let $s_i := 1 - (c + \bw^T\bx_i)$ with $\|\bw\|>0$, and let $\Phi(\cdot)$ and $\phi(\cdot)$ represent the cdf and pdf of a $\mathcal{N}(0,1)$ distribution. Then, we have:
\begin{equation}
\nonumber
\begin{aligned}
 E_{\bx\sim\mathcal{N}(\bx_i, \sigma^2 I)} \left[\ell(C(\bx), y=1)\right]
 = s_i\cdot\Phi\left(\frac{s_i}{\sigma\|\bw\|} \right) + \sigma\|\bw\|\cdot\phi\left(\frac{s_i}{\sigma\|\bw\|} \right).
\end{aligned}
\end{equation}
%Here, $\Phi(\cdot)$ and $\phi(\cdot)$ represent the cdf and pdf of a $\mathcal{N}(0,1)$ distribution.
\label{thm:eloss}
\end{theorem}
% The proof is similar to Theorem~1 of~\cite{chakrabarti_robust_2022}.
The proof is shown in Appendix~\ref{sec:app:proofs}.
Using Theorem~\ref{thm:eloss}, the loss function in Equation~\ref{eq:loss} has a closed-form formula.
Furthermore, the objective is convex since it is the expectation over a convex loss.
Hence, we can minimize Equation~\ref{eq:loss} using any standard convex solver.

% \txtred{\emph{Remark on the Gaussian distribution:}

\begin{remark}
In Theorem~\ref{thm:eloss}, the Gaussian serves as a tractable regularizer that yields a closed-form expected loss.
We do not sample from the high-dimensional Gaussian; thus, we avoid approximation errors.
Also, the simplicity of $\mathcal{C}(\cdot)$ makes scoring ligands extremely fast.
\end{remark}

% centered at each positive ligand. The Gaussian enters the algorithm only
% through the closed-form expected loss in Theorem~\ref{thm:eloss}, where
% it is integrated away analytically. Consequently, concerns about the
% chemical plausibility of continuous perturbations on discrete ECFP
% fingerprints do not apply: no perturbed fingerprint is ever instantiated,
% scored, or used for training. The Gaussian serves purely as a mathematical
% device to produce a tractable regularizer. Alternative discrete distributions
% over the fingerprint space could be more chemically interpretable, but to our
% knowledge none admit a closed-form expected loss; approximating the
% expectation by sampling would reintroduce the approximation error that the
% closed-form solution is designed to avoid, and the required sample count
% grows rapidly with the dimensionality of the feature space.}

% \medskip\noindent
% {\bf Combining classifiers for the top ligands:}
\subsection{Combining Classifiers for Multiple Top Ligands}
Suppose we have trained the classifiers $\mathcal{C}_i$ for all $i\in\mathcal{S}^+$.
Now, we score each untested ligand (the set $\mathcal{D}_{rest}$) via
\begin{align}
\text{score}(\bx_j) &:=  \sum_{i\in\mathcal{S}^+} \alpha^{p_i}\cdot \mathcal{C}_i(\bx_j),
\label{eq:score}
\end{align}
where $p_i$ is the PIC of ligand~$i$, and $\alpha\geq 1$ is a model parameter.
In other words, ligand $j\in\mathcal{D}_{rest}$ is predicted to have a high PIC if it is scored highly by one or more of the classifiers $\mathcal{C}_i$ for $i\in\mathcal{S}^+$.
Intuitively, the classifiers identify patterns that differentiate the top ligands from the rest.
Hence, the weighted sum selects ligands that possess several of these differentiating factors.
The weights $\alpha^{p_i}$ ensure that better ligands in $\mathcal{S}^+$ are given more importance.
The ligands in $\mathcal{D}_{rest}$ with the highest scores are selected for testing in the next DMTA cycle.

% {\bf Combining classifiers for the top ligands:}
% Suppose we have trained the classifiers $\mathcal{C}_i$ for all $i\in\mathcal{S}^+$.
% Now, from the set $\mathcal{D}_{rest}$ of untested ligands, we need to select a few ligands to test in the next DMTA cycle.

% We select the ligands with the highest weighted score, computed as follows.
% For a ligand $j\in\mathcal{D}_{rest}$ with embedding $\bx_j$, we calculate the weighted sum $\sum_{i\in\mathcal{S}^+} \alpha^{p_i}\cdot \mathcal{C}_i(\bx_j)$, where each classifier's score on $\bx_j$ is weighted by the factor $\alpha^{p_i}$.
% Here, $p_i$ is the PIC of ligand~$i$, and $\alpha\geq 1$ is a model parameter. 
% In other words, ligand $j$ is predicted to have a high PIC if it is scored highly by one or more of the classifiers $\mathcal{C}_i$ for $i\in\mathcal{S}^+$.
% Intuitively, the classifiers identify patterns that differentiate the top ligands from the rest.
% Hence, the weighted sum selects ligands that possess several of these differentiating factors.
% The weights $\alpha^{p_i}$ ensure that better ligands in $\mathcal{S}^+$ are given more importance.

\begin{algorithm}[t]
    \small
    \begin{algorithmic}[1]
    \Function{\ourmethod}{$\mathcal{D}_{seen}$, $\mathcal{D}_{rest}$, $\sigma, n_{max}, \alpha, \beta, p^+$}
    \LeftComment $\mathcal{D}_{seen}=\{(\bx_i, p_i)\}$ sorted in decreasing order of PICs $p_i$
    \State $m \gets |\{p_i\geq p^+\mid (\bx_i, p_i)\in\mathcal{D}_{seen}\}|$
    \State $n^+\gets \min(n_{max}, m)$ %\Comment{$n^+=|\mathcal{S}^+|$}
    \State $\mathcal{S}^- \gets \{\bx_i\mid (\bx_i, p_i)\in\mathcal{D}_{seen}, i>\max(\lceil\beta|\mathcal{D}_{seen}|\rceil, m)\}$
    %\State $n^+\gets \min(n_{max}, \max(\beta|\mathcal{D}_{seen}|, m))$ \Comment{$n^+=|\mathcal{S}^+|$}
    %\State $\mathcal{S}^- \gets \{\bx_i\mid (\bx_i, p_i)\in\mathcal{D}_{seen}, i>n^+\}$
    \ForAll{$i\in \{1\ldots n^+\}$}
        \State $\mathcal{C}_i \gets \text{Classify}(\{\bx_i\} \text{ versus } \mathcal{S}^-\mid \sigma)$\Comment{(Eq.~\ref{eq:loss})}\label{alg:ourmethod:Ci}
    \EndFor
    \State $T\gets \left\{\sum_{i\leq n^+} \alpha^{p_i}\cdot \mathcal{C}_i(\bx_j)\mid \bx_j\in\mathcal{D}_{rest}\right\}$\label{alg:ourmethod:score} \Comment{(Eq.~\ref{eq:score})}
    \State\Return top scoring ligands from $T$
    \EndFunction
    \end{algorithmic}
    \caption{\ourmethod}
    \label{alg:ourmethod}
\end{algorithm}

%\medskip\noindent
\subsection{Overall Algorithm}
Algorithm~\ref{alg:ourmethod} shows how \ourmethod selects ligands for one DMTA cycle.
We have a set $\mathcal{D}_{seen}$ of previously tested ligands (whose PICs are known) and the untested set $\mathcal{D}_{rest}$.
From $\mathcal{D}_{seen}$, we select the top ligands with PIC$\geq p^+$, keeping at most $n_{max}$ of these for the positive class.
%some chosen threshold $p^+$.
%If $m$, and keep the top $n_{max}$ of these for the positive class.
%We also ensure that the positive class has at least the top $\beta$-fraction of all seen ligands, and at most $n_{max}$ ligands.
For the negative class, we keep all ligands except those with PIC above $p^+$ or those among the top-$\beta$ fraction of ligands.
%All other ligands in $\mathcal{D}_{seen}$ form the negative set $\mathcal{S}^-$.
Then, we train the classifiers $\mathcal{C}_i$ ($i\in\{1, \ldots, n^+\}$) in step~\ref{alg:ourmethod:Ci}, and score the ligands in $\mathcal{D}_{rest}$ in step~\ref{alg:ourmethod:score}.
The top-scoring ligands from $\mathcal{D}_{rest}$ are then selected for testing in the next DMTA cycle.
Once lab tests reveal their PICs, we add them to $\mathcal{D}_{seen}$.
These become inputs for the next DMTA cycle.
Our novel combination of Gaussian regularization for robustness and the simple form of the classifier and scoring function enables \ourmethod to work even with limited data while scoring candidate ligands $10x$ faster than competing methods.

{\bf Implementation details:}
We standardize the scores output by each $\mathcal{C}_i$ for all ligands in $\mathcal{D}_{rest}$ to zero mean and unit variance (step~\ref{alg:ourmethod:Ci}).
%We find that we get better results if we limit the summation in step~\ref{alg:ourmethod:score} to only the $C_i$s with $p_i\geq p^+$.
Finally, we noticed that if one ligand with a very high PIC is found in the early cycles, it can dictate the choice of ligands for several future cycles.
%This is because that single ligand's scores are given exponentially more weight in step~\ref{alg:ourmethod:score}.
To improve reliability, we limit each ligand $i\in\mathcal{D}_{seen}$ to help select at most $10$ other ligands across all cycles.
In our experiments, we use $\alpha=5, \sigma=1, \beta=0.05, n_{max}=20,$ and $p^+=7$.
In Section~\ref{sec:exp:sensitivity}, we show that apart from $\sigma$, which controls robustness, our results are insensitive to all hyper-parameters.
\section{Experiments}
\label{sec:exp}

We ran experiments to answer the following questions: 
(a)~How quickly do \ourmethod and competing methods find enough ligands with the desired PIC?
(b)~How much computation do they need to score candidate ligands?
(c)~How robust is \ourmethod to its model parameters and choice of ligand embedding?
(d)~How much does each component of \ourmethod contribute to its performance?

\begin{table*}[tbp]
\centering
{\footnotesize
\begin{tabular}{c|ccccc|ccccc}
\toprule
 & \multicolumn{5}{c|}{{\bf Average (Top-10)}} & \multicolumn{5}{c}{{\bf Min (Top-3)}}\\\cline{2-11}
 & \multicolumn{5}{c|}{\textbf{Target PIC}}  & \multicolumn{5}{c}{\textbf{Target PIC}}\\
{\bf Method} & 7.0 & 7.5 & 8.0 & 8.5 & 9.0 & 7.0 & 7.5 & 8.0 & 8.5 & 9.0 \\
\midrule\midrule
Random 	& 21 	& 35 	& 76 	& 215 	& 397 	& 13 	& 21 	& 46 	& 140 	& 348 \\
XGBoost 	& 21 	& 33 	& 55 	& 104 	& 249 	& 13 	& 22 	& 39 	& 82 	& 186 \\
MLP 	& {\bf 17} 	& 25 	& 45 	& 108 	& 288 	& {\bf 12} 	& {\bf 18} 	& 31 	& 78 	& 199 \\
XGBRegressor 	& 18 	& 27 	& 46 	& 103 	& 265 	& {\bf 12} 	& 20 	& 34 	& 78 	& 201 \\
\midrule
TabPFN & 18 & 26 & 48 & 122 & 276 & {\bf 12} & 19 & 35 & 94 & 217 \\
TabM & {\bf 17} &  27 &  59 & 141 & 294 & {\bf 12} &  20 &  41 & 115 & 249 \\
\midrule
GP-M & {\bf 17}  & 27 &   53 & 136 & 290 & {\bf 12} &  21 &  41 & 112 & 235 \\
GP-UCB & {\bf 17} &  28 &  51 & 126 & 277 & {\bf 12}  & 22 &  38 &  99 & 217 \\
GP-EI & 18  & 27 &  44  & {\bf 92} & {\bf 228} & 13  & 20 &  33  & {\bf 70} & 176 \\
GP-PI & {\bf 17} &  25 &  41  & 99 & 251 & {\bf 12} &  19  & 33 &  80 & 184 \\
\midrule
{\bf \ourmethod} 	& {\bf 17} 	& {\bf 24} 	& {\bf 40} 	& {\bf 92} 	& 257 	& {\bf 12} 	& {\bf 18} 	& {\bf 29} 	& 71 	& {\bf 168} \\
\bottomrule
\end{tabular}
}
\caption{{\textbf{Median value of mean-ligands-to-target (MLT) over $100$ proteins (lower is better):}}
\ourmethod requires the fewest or nearly the fewest ligand tests to reach the target PICs in most cases.
}
\label{tbl:medianMLT}
\end{table*}

\subsection{Experimental Setup}
\label{sec:exp:setup}
We constructed a new 1.5M-entry ligand–protein interaction (LPI) dataset from PubChem, yielding $\sim$3.5M datapoints when combined with BindingDB and Davis datasets~\citep{Gilson2015BindingDBI2, Davis2011ComprehensiveAO}. We excluded compounds with PIC $<5$ across all proteins, as they provide little signal for learning. Consequently, our dataset is not directly comparable to high-throughput screening (HTS) hit rates, which include many uniformly inactive compounds and are designed for single-shot screening. In contrast, iterative screening methods like \ourmethod achieve higher hit rates~\citep{paricharak_analysis_2016}. Our curation provides a more informative and balanced benchmark for fair comparison across algorithms.

% We constructed a new 1.5M-entry ligand–protein interaction (LPI) dataset from PubChem. Combined with the BindingDB and Davis datasets, this yields approximately 3.5M LPI measurements. During dataset construction, we excluded compounds with PIC $< 5$ across all proteins, as they provide little informative signal for learning. As a result, our dataset is not directly comparable to high-throughput screening (HTS) hit rates, which include a large fraction of uniformly inactive compounds and are meant for single-shot screening.
% Iterative screening methods like \ourmethod have higher hit rates~\citep{paricharak_analysis_2016}.
% Instead, our curation is designed to create a more informative and balanced evaluation setting, enabling fair and meaningful comparisons between algorithms.

We represented ligands as $2,048$-dimensional ECFP, $167$-dimensional MACCS, and $600$-dimensional ChemBERTa embeddings~\cite{rogers_extended-connectivity_2010, durant_reoptimization_2002, chithrananda2020chemberta}.
We use ECFP as our primary molecular representation throughout
the main experiments, since it performs best on downstream property-prediction tasks~\cite{praski2025benchmarking}.

We simulated early-stage drug discovery on $100$ proteins with the most data.
For each protein, we ran the following experiment.
Let $A$ be all the ligands associated with that protein.
In other words, our dataset contains the PICs for all ligands in $A$.
We keep these PICs hidden from the algorithm until it specifically tests for ligands.
In the first iteration, \ourmethod randomly selects $b$ ligands from $A$.
The PICs of these selected ligands are revealed to \ourmethod.
With this information, \ourmethod selects another $b$ ligands (not at random).
These are tested in the next iteration, and so on.
In this way, each iteration represents a simulated DMTA cycle consisting of $b$ ligand tests (we use $b=10$, but other values yield similar results).

The iterations end when the top-$10$ ligands found so far have an average PIC $\geq t$ for $t\in\{7.5, 8, 8.5, 9\}$.
Alternatively, we stop once the top-$3$ ligands each have PIC $\geq t$.
We call these two endpoints {\em average top-10} and {\em min top-3} respectively.
We note that ligands with PIC$\geq 9.5$ are too rare ($<0.07\%$ of the ligands) to get reliable results.

\begin{table*}[tbp]
\centering
{\small
\begin{tabular}{c|ccccc|ccccc}
\toprule
\textbf{How often} & \multicolumn{5}{c|}{\textbf{Average (Top-10)}} & \multicolumn{5}{c}{\textbf{Min (Top-3)}} \\\cline{2-11}
\textbf{is a method} & \multicolumn{5}{c|}{\textbf{Target PIC}}  & \multicolumn{5}{c}{\textbf{Target PIC}}\\
\textbf{significantly better?} & 7.0 & 7.5 & 8.0 & 8.5 & 9.0 & 7.0 & 7.5 & 8.0 & 8.5 & 9.0 \\
\midrule\midrule
\ourmethod 	& {\bf 95\% } 	& {\bf 94\% } 	& {\bf 91\% } 	& {\bf 68\% } 	& {\bf 24\%} 	& {\bf 20\%} 	& {\bf 61\% } 	& {\bf 80\% } 	& {\bf 75\% } 	& {\bf 34\%} \\
vs. XGBoost 	& 1\% 	& 1\% 	& 2\% 	& 1\% 	& 13\% 	& 0\% 	& 1\% 	& 0\% 	& 1\% 	& 2\% \\
% Neither 	& 4\% 	& 5\% 	& 7\% 	& 31\% 	& 63\% 	& 80\% 	& 38\% 	& 20\% 	& 24\% 	& 64\% \\
\midrule
\ourmethod 	& {\bf 18\%} 	& {\bf 33\%} 	& {\bf 60\% } 	& {\bf 54\% } 	& 41\% 	& 0\% 	& 3\% 	& {\bf 7\%} 	& {\bf 27\%} 	& {\bf 33\%} \\
vs. MLP 	& 3\% 	& 5\% 	& 6\% 	& 6\% 	& 15\% 	& 0\% 	& 3\% 	& 2\% 	& 5\% 	& 10\% \\
%Neither 	& 79\% 	& 62\% 	& 34\% 	& 40\% 	& 44\% 	& 100\% 	& 94\% 	& 91\% 	& 68\% 	& 57\% \\
\midrule
\ourmethod 	& {\bf 78\% } 	& {\bf 79\% } 	& {\bf 69\% } 	& {\bf 43\%} 	& {\bf 28\%} 	& {\bf 2\%} 	& {\bf 18\%} 	& {\bf 36\%} 	& {\bf 28\%} 	& {\bf 24\%} \\
vs. XGBRegressor 	& 1\% 	& 0\% 	& 2\% 	& 4\% 	& 10\% 	& 0\% 	& 0\% 	& 1\% 	& 2\% 	& 7\% \\
%Neither 	& 21\% 	& 21\% 	& 29\% 	& 53\% 	& 62\% 	& 98\% 	& 82\% 	& 63\% 	& 70\% 	& 69\% \\
\midrule
%\vspace{0.2em}

\ourmethod & {\bf 23\%} & {\bf 43\%} & {\bf 46\%} & {\bf 56\%} & {\bf 56\%} & {\bf 3\%} & {\bf 6\%} & {\bf 10\%} & {\bf 30\%} & {\bf 51\%}\\
vs TabM & 4\% & 2\% & 2\% & 2\% & 4\% & 1\% & 2\% & 1\% & 2\% & 1\%\\
\midrule

\ourmethod & {\bf 9\%} & {\bf 16\%} & {\bf 17\%} & {\bf 44\%} & {\bf 43\%} & {\bf 2\%} & {\bf 7\%} & {\bf 8\%} & {\bf 24\%} & {\bf 34\%}\\
vs. GP-M & 8\% & 5\% & 3\% & 6\% & 6\% & 1\% & 1\% & 1\% & 2\% & 3\%\\
\midrule

\ourmethod & 6\% & {\bf 12\%} & {\bf 13\%} & {\bf 26\%} & {\bf 32\%} &  1\% & {\bf 5\%} & {\bf 6\%} & {\bf 13\%} & {\bf 22\%}\\
vs GP-UCB & {\bf 8\%} & 8\% & 11\% & 8\% & 9\% & {\bf 2\%} & 1\% & 1\% & 2\% & 4\%\\
\midrule

\ourmethod & {\bf 54\%} & {\bf 59\%} & {\bf 42\%} & {\bf 23\%} & 15\% & 2\% & {\bf 10\%} & {\bf 16\%} & {\bf 17\%} & 11\%\\
vs GP-EI & 2\% & 3\% & 7\% & 8\% & {\bf 22\%} & 2\% & 3\% & 2\% & 5\% & {\bf 15\%}\\
\midrule

\ourmethod & 2\% & {\bf 9\%} & {\bf 13\%} & {\bf 18\%} & 21\% & 1\% & {\bf 2\%} & {\bf 6\%} & {\bf 13\%} &  13\%\\
vs GP-PI & {\bf 9\%} &  7\% &  9\% &  9\% &  {\bf 23\%} &  {\bf 2\%} &  1\% &  3\% &  2\% &  13\%\\

\midrule
\midrule
{\ourmethod (MACCS)} & {\bf 8\%} & {\bf 9\%} & {\bf 12\%} & {\bf 14\%} & {\bf 24\%} & 1\% & 2\% & 2\% & {\bf 8\%} & {\bf 18\%}\\
vs TabPFN (MACCS) & 5\% & 5\% & 8\% & 9\% & 4\% & 1\% & 2\% & 2\% & 4\% & 0\%\\
\midrule

{\ourmethod (ECFP)} & {\bf 42\%} & {\bf 46\%} & {\bf 49\%} & {\bf 45\%} & {\bf 43\%} & 1\% & {\bf 8\%} & {\bf 11\%} & {\bf 28\%} & {\bf 36\%}\\
vs TabPFN (MACCS) & 2\% & 1\% & 1\% & 3\% & 1\% & 1\% & 1\% & 1\% & 3\% & 1\%\\

\bottomrule
\end{tabular}
}
\caption{{\textbf{ Head-to-head comparisons:}}
%For each protein, we check whether one method reaches the endpoint earlier than the other method does, significantly more often than by chance.
We report the percentage of proteins for which one method reaches the endpoint earlier than the other method, significantly more often than by chance.
%is reliably better than the other.
%Since TabPFN does not run for ECFP, we compare it against \ourmethod on the MACCS embedding.
TabPFN runs only with MACCS, not ECFP, so we compare against \ourmethod run with both embeddings.
For our main target PICs of $8$ and $8.5$, \ourmethod is reliably better for more proteins than any other method.
}
\label{tbl:H2HtargetPICfreq}
\vspace{-0.5em}
\end{table*}

\textbf{Metrics:}
For each experiment, we calculated the $\text{\bf ligands-to-target}(t)$, which denotes the number of ligand tests needed to hit an endpoint at PIC $t$.
If the endpoints are not reached within $400$ ligand tests, we declare failure and set the ligands-to-target$=400$.
To remove the effect of the random initialization in the first iteration, we averaged this value over $50$ repetitions.
We call the result mean-ligands-to-target, or $\text{MLT}(t)$.
Smaller values of $\text{MLT}(t)$ imply faster drug discovery.

\textbf{Competing methods:}
We consider the following types of competing methods.

{\bf (a) Bayesian optimization:} 
We use a Gaussian Process (GP) with a Tanimoto kernel, which has been found to work best for molecules~\cite{griffiths2024gauche}.
The GP gives mean and variance estimates for our current belief about each ligand's PIC.
In each DMTA cycle, we select ligands with the highest mean (GP-M), or mean plus standard deviation (GP-UCB), or the expected improvement in PIC (GP-EI), or the probability of improvement (GP-PI).

\noindent
{\bf (b) Deep learning:}
Since we have limited training data, we consider two recent state-of-the-art methods for such settings, namely TabM~\cite{gorishniy_tabm_2025} and TabPFN~\cite{hollmann2025tabpfn}.

\noindent
{\bf (c) Standard tools:}
We also tested XGBoost, Multilayer Perceptrons (MLP), and XGBRegressor.
As a control, we added a method, named Random, that randomly selects ligands in each DMTA cycle.

\subsection{Comparisons Between Methods}
\label{sec:exp:comparison}

% We first plot the ligands-to-target for five proteins that exemplify our overall results
% Figure~\ref{fig:best5}.
\paragraph{Overall comparison of mean ligands-to-target (MLT):}
Table~\ref{tbl:medianMLT} shows, for each method, the median MLT over $100$ proteins.
%We observe the following:
{\bf \ourmethod is the fastest, or nearly the fastest, to the target PIC in almost all cases.}
%Thus, aggregated over all $100$ proteins, \ourmethod outperforms all other methods.
\ourmethod needs only $\mathbf{29-40}$ ligand tests to reach a PIC of $\mathbf{8}$, for both the {\em average top-10} and {\em min top-3} endpoints.
For a target PIC of $8.5$, we only need $71-92$ ligand tests.

%As a result, the \emph{race-to-8} concludes before \ourmethod needs to gather much training data from ligand tests.
    
% We note that TabPFN did not run with the 2048-dimensional ECFP ligand embeddings, so we use the MACCS embeddings. All other methods used the ECFP embeddings.

% \begin{table*}
% \centering

% \label{tab:tabpfn_matched}
% \small
% \begin{tabular}{lccccc}
% \toprule
% Target PIC              & 7.0 & 7.5 & 8.0 & 8.5 & 9.0 \\
% \midrule
% \multicolumn{6}{l}{\emph{Matched: both methods use MACCS}} \\
% SPADE (MACCS) is better  & \textbf{8\%}  & \textbf{9\%}  & \textbf{12\%} & \textbf{14\%} & \textbf{24\%} \\
% TabPFN (MACCS) is better & 5\%  & 5\%  & 8\%  & 9\%  & 4\% \\
% \midrule
% \multicolumn{6}{l}{\emph{Mismatched: SPADE uses ECFP, TabPFN uses MACCS}} \\
% SPADE (ECFP) is better   & \textbf{42\%} & \textbf{46\%} & \textbf{49\%} & \textbf{45\%} & \textbf{43\%} \\
% TabPFN (MACCS) is better & 2\%  & 1\%  & 1\%  & 3\%  & 1\% \\
% \bottomrule
% \end{tabular}
% \caption{\txtred{SPADE vs.\ TabPFN head-to-head under matched (both MACCS) and
% mismatched (SPADE with ECFP, TabPFN with MACCS) embedding conditions. Even
% when both methods are restricted to MACCS (matched), SPADE outperforms TabPFN
% across all target PICs.}}
% \end{table*}

% Our robust classifier is key in \ourmethod's outperformance in such data-sparse settings.

\paragraph{Head-to-head comparisons:}
%Next, we directly compared \ourmethod with competing methods.
For each protein, we tracked which method reached the endpoint first over the 50 trials.
If two methods are similar, each should finish first about equally often.
A method is reliably better for a protein if it finishes first significantly more than half the time ($p<0.1$; results are similar for $p<0.05$).
Table~\ref{tbl:H2HtargetPICfreq} shows the fraction of proteins where \ourmethod or a competitor is reliably better.
% No single method dominates across all proteins.
% Between any two methods, we choose the one that is reliably better for more proteins.
% By this measure, 
{\bf \ourmethod beats every competitor for the main target PICs of 8 and 8.5.}
GP-PI is the closest competitor, but \ourmethod is 10x faster (as shown later). 

Table~\ref{tbl:percentImprov} compares the MLT for proteins where one method is reliably better than another.
% We now focus exclusively on proteins where one method is reliably better than another.
% For these proteins, we compare the mean-ligands-to-target (MLT) of the two methods (Table~\ref{tbl:percentImprov}).
For the {\em average top-10} metric with target PICs $8$ or $8.5$, \ourmethod outperforms Bayesian optimization by $8\%-31\%$, deep learning by $7\%-32\%$, and standard classifiers and regressors by $13\%-25\%$.
Thus, {\bf \ourmethod needs $\mathbf{7\%-32\%}$ fewer ligand tests than its competitors.}

% \paragraph{Need for robustness:}
% Among the top methods, the biggest differences are for PICs 8 and 8.5.
% For lower PICs, all methods reach the endpoints quickly, and any differences are washed out by randomness.
% For higher PICs, all methods have enough data to select good ligands, so the gaps shrink.
% However, our focus is on the race-to-8.
% Here, \ourmethod's robustness helps it dominate.

\begin{table*}[tbp]
\centering
{\footnotesize
\begin{tabular}{c|ccccc|ccccc}
\toprule
 {\bf Median} & \multicolumn{5}{c|}{{\bf Average (Top-10)}} & \multicolumn{5}{c}{{\bf Min (Top-3)}}\\\cline{2-11}
 {\bf improvement of} & \multicolumn{5}{c|}{\textbf{Target PIC}}  & \multicolumn{5}{c}{\textbf{Target PIC}}\\
{\bf \ourmethod over} & 7.0 & 7.5 & 8.0 & 8.5 & 9.0 & 7.0 & 7.5 & 8.0 & 8.5 & 9.0 \\
\midrule\midrule

XGBoost 	& 18\% & 24\% & 25\% & 16\% & 3\% & 22\% & 21\% & 23\% & 16\% & 14\% \\
MLP 	& 6\% &  10\% &  18\% &  22\% &  12\% &  0\% &  8\% & 17\% &  18\% &  22\% \\
XGBRegressor 	& 7\% &  11\% &  13\% &  18\% &  11\% &  10\% &  19\% &  13\% &  20\% &  24\% \\
\midrule
TabPFN &  4\% &  13\% &  7\% &  12\% &  8\% &  -9\% &  -15\% &  -19\% &  6\% &  12\% \\
TabM &  4\% &  17\% &  32\% &  32\% &  19\% &  42\% &  43\% &  39\% &  40\% &  27\% \\
\midrule
GP-M &  2\% &  26\% &  20\% &  31\% &  12\% &  45\% &  45\% &  19\% &  38\% &  28\% \\
GP-UCB &  0\% &  14\% &  16\% &  24\% &  11\% &  -8\% &  44\% &  17\% &  28\% &  26\% \\
GP-EI &  11\% &  13\% &  8\% &  13\% &  -7\% &  -14\% &  14\% &  14\% &  12\% &  -1\% \\
GP-PI &  -5\% &  12\% &  8\% &  17\% &  -2\% &  -59\% &  43\% &  22\% &  19\% &  4\% \\

\bottomrule
\end{tabular}
}
\caption{
{\textbf {Percent lift of \ourmethod over competing methods (higher is better):}}
For proteins where one method is reliably better, we compute the improvement in \ourmethod's MLT over the competing method. 
\ourmethod has a median improvement of $8\%-32\%$ over competing methods for our primary target PICs of $8$ and $8.5$ under the {\em average top-10} metric.
}
\label{tbl:percentImprov}
\vspace{-1em}
\end{table*}
\begin{figure*}
\centering
        \begin{subfigure}{0.3\textwidth}
            \centering
            \includegraphics[width=\textwidth]{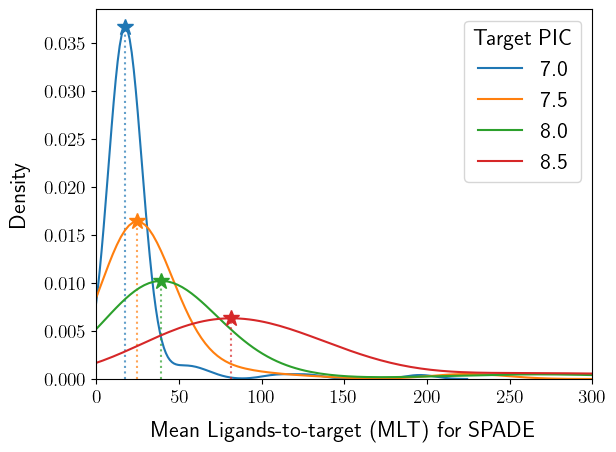}
            \caption{\ourmethod's MLT distribution}
            \label{fig:detail:ourtimeplot}
        \end{subfigure}
    \hfill
        \begin{subfigure}{0.68\textwidth}
            \centering
            \includegraphics[width=\textwidth]{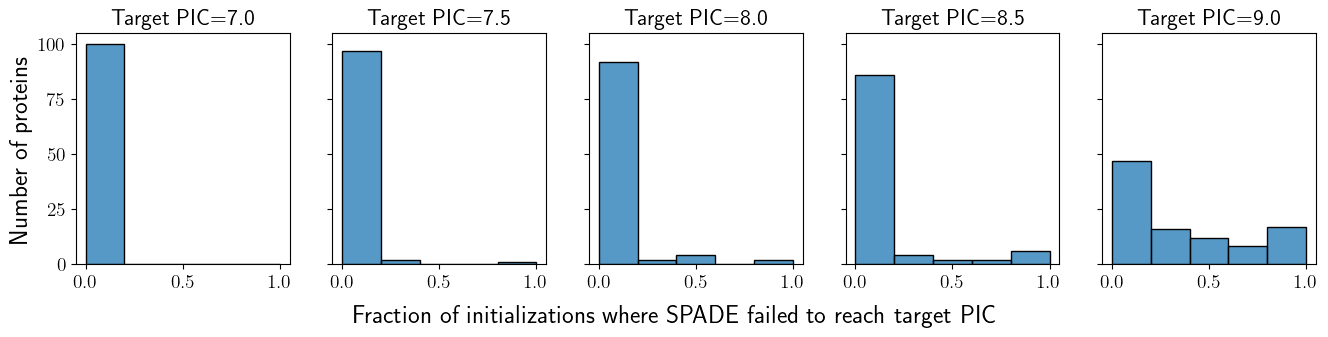}\vspace{-0.5em}
            \caption{Out-of-reach target PICs}
            
            \label{fig:detail:wefailed} 
        \end{subfigure}

  \caption{{\textbf{ Detailed analysis of \ourmethod's performance:}}
(a) As the target PIC increases, the distribution of \ourmethod's mean ligands-to-target (MLT) shifts to the right and has higher variance.
%as the target PIC increases, since we need more ligand tests to achieve the target.
%over $100$ proteins using the {\em average top-10} metric.
%As the target PIC increases, the MLT distribution shifts to the right, since we need more ligand tests to achieve the target.
%The distribution also widens due to correlations among the selected ligands over the DMTA cycles.
(b)~\ourmethod's failures to reach a PIC occur most for target PIC$=9$, which are very rare (less than $0.5\%$ of the ligand for the median protein).
Detailed explanations are in the text.
% Since we stop after $400$ ligand tests, we sometimes fail to hit a target PIC$\geq 9$.
    }
    \label{fig:detail}
    \vspace{-1em}
\end{figure*}
\begin{figure*}[tbp]
    \centering
    \includegraphics[width=\textwidth]{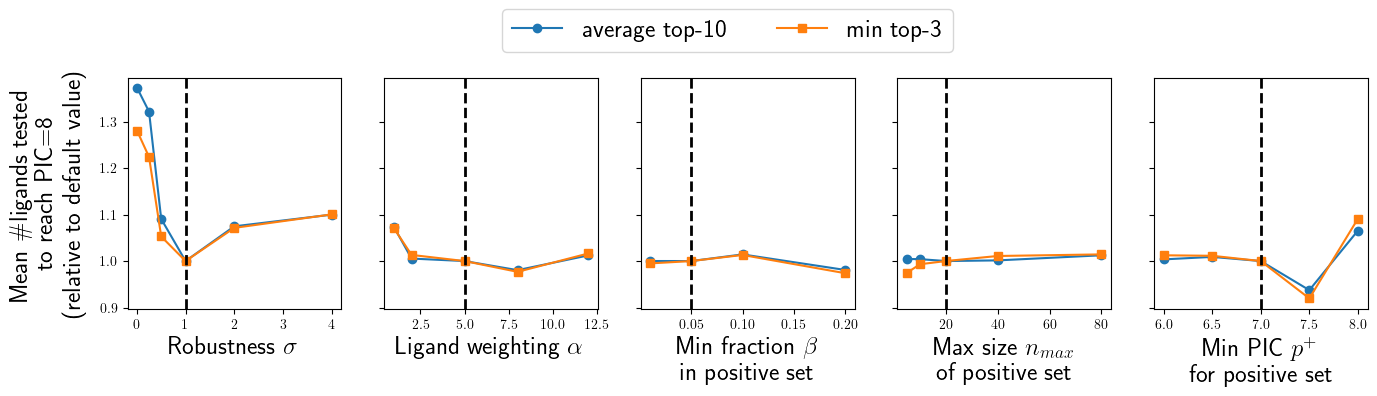}
    \caption{{\textbf{Sensitivity analysis:}} \ourmethod's performance is only sensitive to the robustness parameter $\sigma$ used in Equation~\ref{eq:loss}.
    }
    \label{fig:sensitivity}
    \vspace{-1em}
\end{figure*}

\begin{table}[t]
\centering
\footnotesize
\begin{subtable}{0.28\textwidth}
\centering
%\scriptsize
% \begin{tabular}{c|ccccc}
%     \multicolumn{6}{c}{{\bf $\Delta$-MLT from MACCS ($d=167$) to ECFP ($d=2048$)}}\\
%     \toprule
%         & \multicolumn{5}{c}{{\bf Target PIC}} \\
% {\bf Metric}   & 7.0 & 7.5 & 8.0 & 8.5 & 9.0\\
% \toprule
% {\em Average top-10} & 5\% & 12\% & 19\% & 29\% & 5\% \\
% {\em Min top-3}      & 0\% & 11\% & 17\% & 28\% & 25\% \\
% \bottomrule
% \end{tabular}
% \caption{Sensitivity to embedding dimension.}
% \label{tbl:embedding}

\begin{tabular}{lr}
\toprule
\textbf{Method} & \textbf{Time (s)} \\
\midrule
\textbf{\ourmethod}   & $\mathbf{4\pm0.1}$ \\
GP-(any)              & $39\pm0.2$  \\
TabM                  & $97\pm 8.6$ \\
TabPFN                & $6096\pm548$ \\
\bottomrule
\end{tabular}
\caption{Time to score $10^6$ ligands}
\label{tbl:time}
\end{subtable}
\hfill
\begin{subtable}{0.35\textwidth}
\centering
%\scriptsize
\begin{tabular}{lc}
\toprule
\textbf{\# ligand tests in} & \textbf{Relative diff.}\\
\textbf{a DMTA cycle ($b$)} & \textbf{from default} \\
\midrule
$b = 5$   & $0.3\% \pm 1.05\%$ \\
$b = 10$  & (default) \\
$b = 20$  & $0.4\% \pm 1.8\%$ \\
\bottomrule
\end{tabular}
\caption{Sensitivity to batch size}
\label{tab:k_sensitivity}
\end{subtable}
\hfill
\begin{subtable}{0.32\textwidth}
\centering
%\scriptsize
\begin{tabular}{lc}
\toprule
\textbf{Condition} & \textbf{Effect} \\
\midrule
No robustness  & $37\%$ worse MLT \\
No exp.\ weight & $10\%$ worse MLT \\
No 10-ligand limit & PIC diff. $<0.04$ \\
\bottomrule
\end{tabular}
\caption{Ablation study}
\label{tbl:ablation}
\end{subtable}

\caption{\textbf{Summary of additional analyses for \ourmethod.}}
\label{tab:combined_three}

\end{table}

\subsection{Detailed Analysis of \ourmethod}
\label{sec:exp:detailed}

\textbf{Fast ligand discovery, with wider tails for harder targets:}
Figure~\ref{fig:detail:ourtimeplot} shows that \ourmethod's MLT distribution shifts to the right as the target PIC increases, since we need more ligand tests to achieve the target.
The variance also increases for the higher PICs, due to correlated ligand selection across the DMTA cycles.
The effect of such correlations compounds as the number of DMTA cycles increases.
Since we need more cycles for higher target PICs, we see higher variance for them.

\textbf{\ourmethod only fails when the target PICs are especially rare:}
% An initialization fails if \ourmethod does not reach the target within $400$ ligand tests.
For each protein, we count the fraction of times \ourmethod does not reach the target within $400$ ligand tests.
Figure~\ref{fig:detail:wefailed} shows the histogram of this fraction across all proteins.
Almost all observed failures occur when the target PIC is $9$, which is extremely rare in our data ($<0.5\%$ of ligands).
For the target PIC of $8.5$, failures only occur for proteins where fewer than $0.1\%$ of the ligands have PIC$\geq 8.5$ (the typical rate  is $2.7\%$). GP-EI and GP-PI fail at the same rate as SPADE, and GP-UCB, GP-M, and TabM fail at about 45\% higher rate.

\paragraph{Wall-clock time:}
%Finally, we compare methods based on their computational requirements.
The main bottleneck is scoring millions of untested ligands to select the best ones for the next DMTA cycle.
Table~\ref{tbl:time} shows the wall-clock times for scoring 1{,}000 ligands.
{\bf \ourmethod is 10x faster} than its closest competitor (GP-PI).

\begin{table}[t]
\centering
\footnotesize
\setlength{\tabcolsep}{4pt}
\begin{tabular}{l|ccccc|ccccc}
\toprule
 & \multicolumn{5}{c|}{\textbf{ChemBERTa (600 dim.)}} & \multicolumn{5}{c}{\textbf{MACCS (167 dim.)}} \\
 & \multicolumn{5}{c|}{Target PIC} & \multicolumn{5}{c}{Target PIC} \\
\textbf{Method} & 7.0 & 7.5 & 8.0 & 8.5 & 9.0 & 7.0 & 7.5 & 8.0 & 8.5 & 9.0 \\
\midrule
\ourmethod\ is better & 11\% & 16\% & \textbf{17\%} & \textbf{20\%} & 14\%
                      & \textbf{10\%} & \textbf{10\%} & \textbf{16\%} & 13\% & 14\% \\
GP-PI is better       & 2\%  & 5\%  & 8\%  & 15\% & \textbf{29\%}
                      & 3\%  & 6\%  & 8\%  & \textbf{16\%} & \textbf{26\%} \\
\bottomrule
\end{tabular}
\vspace{1em}
\caption{\textbf{\ourmethod\ vs.\ GP-PI under ChemBERTa and MACCS.}
Percentage of proteins where each method is significantly better ($p<0.1$).
\ourmethod\ dominates at lower-to-mid target PICs and is comparable or better at PIC 8.5.}
\label{tab:combined_stacked}
\vspace{-1em}
\end{table}

% \begin{table}[t]
% \centering
% \small
% \begin{tabular}{lc}
% \toprule
% Batch size $k$ & Relative difference in avg.\ top-10 PIC vs.\ $k{=}10$ \\
% \midrule
% $k = 5$   & $0.3\% \pm 1.05\%$ \\
% $k = 10$  & - (reference) \\
% $k = 20$  & $0.4\% \pm 1.8\%$ \\
% \bottomrule
% \end{tabular}
% \vspace{1em}
% \caption{{\textbf {Sensitivity of \ourmethod to the DMTA batch size $k$}}}
% \label{tab:k_sensitivity}
% \end{table}

%\vspace{-1em}
\subsection{Sensitivity Analysis and Ablation Study}
\label{sec:exp:sensitivity}

% This emphasizes the importance of robustness in the {\em race-to-8} problem.

\textbf{Ligand embedding:}
We repeated our experiments using the MACCS and ChemBERTa ligand embeddings.
Both are lower-dimensional than the ECFP embedding.
Table~\ref{tab:combined_stacked} compares \ourmethod\ against GP-PI,
its closest competitor, under the ChemBERTa and MACCS embeddings.
At our primary target PIC of 8,
\ourmethod\ is reliably better than GP-PI on roughly twice as many
proteins under both embeddings. (17\% vs.\ 8\% for ChemBERTa; 16\%
vs.\ 8\% for MACCS).
At PIC 8.5, \ourmethod\ is also comparable or better.
Thus, \ourmethod consistently outperforms other methods in the race-to-8, for all embeddings.

\textbf{Model hyper-parameters:} 
Figure~\ref{fig:sensitivity} shows how \ourmethod's average MLT changes as we vary its hyperparameters.
All results are for our primary target PIC of $8$, and are normalized with respect to the default hyperparameter settings.
Only the robustness parameter $\sigma$ significantly affects performance (Equation~\ref{eq:loss}).
We note that we use the same default value of $\sigma = 1$ for all embeddings (ECFP, MACCS, and ChemBERTa) and all proteins, without per-protein retuning.
{\em \ourmethod does not need cross-validation.}
Since we need fewer than 40 total tests before
reaching target PIC=8, any cross-validation based selection strategy would be
dominated by sampling noise in this low-data setting.

\textbf{Number of ligand tests in each DMTA cycle:}
Table~\ref{tab:k_sensitivity} shows the relative difference in \ourmethod's performance as we vary the number of ligands \txtred{$b$} tested per DMTA cycle.
The average top-10 PIC found by \ourmethod varies only slightly with \txtred{$b$}, confirming
that the method is robust to this hyperparameter. Values are
mean~$\pm$~standard deviation of the relative difference vs.\ \txtred{$b{=}10$}.

% \subsection{Ablation Study}
% \label{sec:exp:ablation}
\textbf{Ablation Study:}
We removed three aspects of \ourmethod: its robustness, the exponential weighting scheme, and the condition that any ligand can be used to help select at most $10$ other ligands.
These correspond to Steps~\ref{alg:ourmethod:Ci} and~\ref{alg:ourmethod:score} of Algorithm~\ref{alg:ourmethod}, and the implementation details in Section~\ref{sec:prop}.
Table~\ref{tbl:ablation} shows that the first two aspects are important to \ourmethod's performance, while the last one is less important.

% \medskip\noindent
% \txtred{{\em Default value of $\sigma$ and its selection:}
% The robustness hyperparameter $\sigma$ acts as a regularizer that controls the
% spread of the Gaussian distribution around each positive sample. To select a
% default value, we tuned $\sigma$ on approximately 10 proteins held out from
% our 100-protein benchmark and selected $\sigma = 1$ for ECFP embeddings. All
% reported results are obtained on the remaining 90 proteins with this fixed
% default. Our sensitivity analysis (Figure~\ref{fig:sensitivity}) further
% confirms that performance under $\sigma = 0.5$ is within 5\% of the $\sigma =
% 1$ default, and that $\sigma = 1$ generalizes across both ECFP and MACCS
% embeddings without per-protein retuning. We do not use a data-driven choice
% for $\sigma$ because in our regime (typically fewer than 40 total tests before
% reaching target PIC=8), any cross-validation based selection strategy would be
% dominated by sampling noise.}

\textbf{Filtering ligands using data for similar proteins:}
Finally, we simulated an experiment where we know of proteins similar to the target protein.
Here, the data from similar proteins is used to filter the set of available ligands before running \ourmethod.
This improves MLT by $18\%-31\%$ (Table~\ref{tab:filter} in Appendix~\ref{sec:app:expdetails}).

\section{Conclusions}
\label{sec:conc}

We set out to rapidly identify ligands with PIC$\geq 8$ for a specific target protein.
These high-quality ligands, however, are the proverbial needles in the vast haystack of all possible ligands.
We also assume no prior knowledge about the protein.
Remarkably, just 40 or so ligand tests suffice for \ourmethod to discover 10 high-quality ligands in this challenging setting.
\ourmethod needs to test $7\%-32\%$ fewer ligands than its competitors.
Moreover, it is also three times faster at scoring ligands than its closest competitor.
This translates to significant cost savings for early-stage drug discovery.

\ourmethod succeeds for two reasons.
First, it focuses only on improving the top-$k$ ligands seen so far.
It does not estimate the PIC distribution across all ligands.
Second, \ourmethod is designed to be robust.
Specifically, in optimizing the model parameters, we minimize the expected loss over a broad distribution rather than the empirical loss on just $k$ top ligands.
This approach helps \ourmethod learn reliable signals from such limited and imbalanced data.

% There are several avenues for future work.
% One direction is to leverage prior knowledge gained from other proteins.
% Such side information could be used to filter the ligand search space before running \ourmethod.
% Another possible extension is to use \ourmethod to guide ligand generative models~\cite{bengio21flow} towards high-PIC ligands.

% For instance, existing methods often aim to generate diverse ligands.
% One could adapt them to generate new ligands likely to have high PICs using \ourmethod.

\newpage
\bibliographystyle{plainnat}
\bibliography{bibliography}

@inproceedings{chakrabarti_robust_2022,
	address = {Vienna, Austria},
	title = {Robust {High}-{Dimensional} {Classification} {From} {Few} {Positive} {Examples}},
	isbn = {978-1-956792-00-3},
	url = {https://www.ijcai.org/proceedings/2022/271},
	doi = {10.24963/ijcai.2022/271},
	language = {en},
	urldate = {2024-04-14},
	booktitle = {Proceedings of the {Thirty}-{First} {International} {Joint} {Conference} on {Artificial} {Intelligence}},
	publisher = {International Joint Conferences on Artificial Intelligence Organization},
	author = {Chakrabarti, Deepayan and Fauber, Benjamin},
	month = jul,
	year = {2022},
	pages = {1952--1958},
}

@article{Swinney2011HowWN,
  title={{How Were New Medicines Discovered?}},
  author={David C. Swinney and Jason Anthony},
  journal={Nat. Rev. Drug Discov.},
  year={2011},
  volume={10},
  pages={507-519},
}

@article{Waring2015AnAO,
  title={{An Analysis of the Attrition of Drug Candidates from Four Major Pharmaceutical Companies}},
  author={Michael J. Waring and John Edmund Arrowsmith and Andrew R. Leach and Paul D. Leeson and Sam Mandrell and Robert M. Owen and Garry Pairaudeau and William D. Pennie and Stephen D. Pickett and Jibo Wang and Owen Wallace and Alexander Weir},
  journal={Nat. Rev. Drug Discov.},
  year={2015},
  volume={14},
  pages={475-486},
}

@article{Sadybekov2021SynthonbasedLD,
  title={{Synthon-Based Ligand Discovery in Virtual Libraries of over 11 Billion Compounds}},
  author={Arman A. Sadybekov and Anastasiia V. Sadybekov and Yongfeng Liu and Christos Iliopoulos-Tsoutsouvas and Xi-Ping Huang and Julie E. Pickett and Blake Houser and Nilkanth Patel and Ngan K. Tran and Fei Tong and Nikolai Zvonok and M. K. Jain and Olena V. Savych and Dmytro S. Radchenko and Spyros P. Nikas and Nicos A. Petasis and Yurii S. Moroz and Bryan L. Roth and Alexandros Makriyannis and Vsevolod Katritch},
  journal={Nature},
  year={2021},
  volume={601},
  pages={452 - 459},
}

@article{Yamanishi2008PredictionOD,
  title={{Prediction of Drug–Target Interaction Networks from the Integration of Chemical and Genomic Spaces}},
  author={Yoshihiro Yamanishi and Michihiro Araki and Alex Gutteridge and Wataru Honda and Minoru Kanehisa},
  journal={Bioinformatics},
  year={2008},
  volume={24},
  pages={i232 - i240},
}

@article{Knowles2003NRDD,
    author = {Jonathan Knowles and Gianni Gromo},
    title = {{Target Selection in Drug Discovery}},
    journal = {Nat. Rev. Drug Discov.},
    year = {2003},
    volume={2},
    pages={63-69},
}

@article{FauberLPI2024,
    author = {Ben Fauber},
    title = {{Accurate Prediction of Ligand-Protein Interaction Affinities with Fine-Tuned Small Language Models}},
    journal = {ArXiv},
    year = {2024},
    volume = {abs/2407.00111v1},
}

@article{PLOWRIGHT201256,
    title = {{Hypothesis Driven Drug Design: Improving Quality and Effectiveness of the Design-Make-Test-Analyse Cycle}},
    journal = {Drug Discov. Today},
    volume = {17},
    number = {1},
    pages = {56-62},
    year = {2012},
    author = {Alleyn T. Plowright and Craig Johnstone and Jan Kihlberg and Jonas Pettersson and Graeme Robb and Richard A. Thompson},
}

@article{LombardinoLoweNRDD2004,
    author = {Joseph G. Lombardino and John A. Lowe III},
    title = {{The Role of the Medicinal Chemist in Drug Discovery — Then and Now}},
    journal = {Nat. Rev. Drug Discov.},
    year = {2004},
    volume = {3},
    pages = {853-862},
}

@article{Oliveira2024InferringMI,
  title={{Inferring Molecular Inhibition Potency with AlphaFold Predicted Structures}},
  author={Pedro F. Oliveira and Rita C Guedes and Andre O Falcao},
  journal={Sci. Rep.},
  year={2024},
  volume={14},
  pages = {8252},
}

@article{Kimber2021DeepLI,
  title={{Deep Learning in Virtual Screening: Recent Applications and Developments}},
  author={Talia B. Kimber and Yonghui Chen and Andrea Volkamer},
  journal={Int. J. Mol. Sci.},
  year={2021},
  volume={22},
  pages={4435},
}

@article{MartinAllAssayMax2QSAR2019,
  author = {Martin, Eric J. and Polyakov, Valery R. and Zhu, Xiang-Wei and Tian, Li and Mukherjee, Prasenjit and Liu, Xin},
  title = {{All-Assay-Max2 pQSAR: Activity Predictions as Accurate as Four-Concentration IC50s for 8558 Novartis Assays}},
  journal = {J. Chem. Inf. Model.},
  volume = {59},
  number = {10},
  pages = {4450-4459},
  year = {2019},
}

@article{Mayr2018LargescaleCO,
  title={{Large-Scale Comparison of Machine Learning Methods for Drug Target Prediction on ChEMBL}},
  author={Andreas Mayr and G{\"u}nter Klambauer and Thomas Unterthiner and Marvin N. Steijaert and J{\"o}rg Kurt Wegner and Hugo Ceulemans and Djork-Arn{\'e} Clevert and Sepp Hochreiter},
  journal={Chem. Sci.},
  year={2018},
  volume={9},
  pages={5441-5451},
}

@article{Martin2011ProfileQSARAN,
  title={{Profile-QSAR: A Novel meta-QSAR Method that Combines Activities across the Kinase Family To Accurately Predict Affinity, Selectivity, and Cellular Activity}},
  author={Eric J. Martin and Prasenjit Mukherjee and David C. Sullivan and Johanna M. Jansen},
  journal={J. Chem. Inf. Model.},
  year={2011},
  volume={51},
  number={8},
  pages={1942-1956},
}

@article{FaulonMisraDTIPred2007,
    author = {Jean-Loup Faulon and Milind Misra and Shawn Martin and Ken Sale and Rajat Sapra},
    title = {{Genome Scale Enzyme–Metabolite and Drug–Target Interaction Predictions Using the Signature Molecular Descriptor}},
    journal = {Bioinformatics},
    volume = {24},
    number = {2},
    pages = {225-233},
    year = {2007},
    issn = {1367-4803},
}

@article{Huang2020DeepPurposeAD,
  title={{DeepPurpose: A Deep Learning Library for Drug–Target Interaction Prediction}},
  author={Kexin Huang and Tianfan Fu and Lucas Glass and Marinka Zitnik and Cao Xiao and Jimeng Sun},
  journal={Bioinformatics},
  year={2020},
  volume={36},
  pages={5545 - 5547},
}

@article{Huang2020MolTransMI,
  title={{MolTrans: Molecular Interaction Transformer for Drug–Target Interaction Prediction}},
  author={Kexin Huang and Cao Xiao and Lucas Glass and Jimeng Sun},
  journal={Bioinformatics},
  year={2020},
  volume={37},
  pages={830 - 836},
}

@article{Li2020MONNAM,
  title={{MONN: A Multi-objective Neural Network for Predicting Compound-Protein Interactions and Affinities}},
  author={Shuya Li and Fangping Wan and Hantao Shu and Tao Jiang and Dan Zhao and Jianyang Zeng},
  journal={Cell Syst.},
  year={2020},
  volumne={10},
  pages={308-322.e11},
}

@article{ztrk2019WideDTAPO,
  title={{WideDTA: Prediction of Drug-Target Binding Affinity}},
  author={Hakime {\"O}zt{\"u}rk and Elif Ozkirimli Olmez and Arzucan {\"O}zg{\"u}r},
  journal={ArXiv},
  year={2019},
  volume={abs/1902.04166},
}

@article{Whitehead2019ImputationOA,
  title={{Imputation of Assay Bioactivity Data Using Deep Learning}},
  author={Thomas M. Whitehead and Benedict W J Irwin and Peter A. Hunt and Matthew D. Segall and Gareth John Conduit},
  journal={J. Chem. Inf. Model.},
  year={2019},
  volume={59},
  issue={3},
  pages={1197-1204},
}

@article{Lee2018DeepConvDTIPO,
    title={{DeepConv-DTI: Prediction of Drug-Target Interactions via Deep Learning with Convolution on Protein sequences}},
    author={Ingoo Lee and Jongsoo Keum and Hojung Nam},
    journal={PLOS Comput. Biol.},
    year = {2019},
    month = {06},
    volume = {15},
    number = {6},
    pages = {e1007129},
}

@article{beyondthehypeDL2017,
  title={{Beyond the Hype: Deep Neural Networks Outperform Established Methods Using a ChEMBL Bioactivity Benchmark Set}},
  author={Eelke B. Lenselink and Niels ten Dijke and Brandon Bongers and George Papadatos and Herman W. T. van Vlijmen and Wojtek Kowalczyk and Adriaan P. IJzerman and Gerard J. P. van Westen },
  journal={J. Cheminform.},
  year={2017},
  volume={9},
  pages={45},
}

@article{Wen2017DeepLearningBasedDI,
  title={{Deep-Learning-Based Drug-Target Interaction Prediction}},
  author={Ming Wen and Zhimin Zhang and Shaoyu Niu and Haozhi Sha and Rui Yang and Yong-Huan Yun and Hongmei Lu},
  journal={J. Proteome Res.},
  year={2017},
  volume={16},
  pages={1401-1409},
}

@article{Kalakoti2022TransDTITL,
  title={{TransDTI: Transformer-Based Language Models for Estimating DTIs and Building a Drug Recommendation Workflow}},
  author={Yogesh Kalakoti and Shashank Yadav and Durai Sundar},
  journal={ACS Omega},
  year={2022},
  volume={7},
  pages={2706 - 2717},
}

@article{gorantla24proteins,
  author = {Gorantla, Rohan and Kubincová, Alžbeta and Weiße, Andrea Y. and Mey, Antonia S. J. S.},
  title = {{From Proteins to Ligands: Decoding Deep Learning Methods for Binding Affinity Prediction}},
  journal = {J. Chem. Inf. Model.},
  volume = {64},
  number = {7},
  pages = {2496-2507},
  year = {2024},
}

@article{Svensson2024HyperPCMRT,
  title={{HyperPCM: Robust Task-Conditioned Modeling of Drug–Target Interactions}},
  author={Emma Svensson and Pieter-Jan Hoedt and Sepp Hochreiter and G{\"u}nter Klambauer},
  journal={J. Chem. Inf. Model.},
  year={2024},
  volume={64},
  pages={2539 - 2553},
}

@article{Chatterjee2023ImprovingTG,
  title={{Improving the Generalizability of Protein-Ligand Binding Predictions with AI-Bind}},
  author={Ayan Chatterjee and Robin Walters and Zohair Shafi and Omair Shafi Ahmed and Michael Sebek and Deisy Morselli Gysi and Rose Yu and Tina Eliassi-Rad and Albert-L{\'a}szl{\'o} Barab{\'a}si and Giulia Menichetti},
  journal={Nat. Commun.},
  year={2023},
  volume={14},
  pages={1989},
}

@article{Thafar2022Affinity2VecDB,
  title={{Affinity2Vec: Drug-Target Binding Affinity Prediction Through Representation Learning, Graph Mining, and Machine Learning}},
  author={Maha A. Thafar and Mona Alshahrani and Somayah Albaradei and Takashi Gojobori and Magbubah Essack and Xin Gao},
  journal={Sci. Rep.},
  year={2022},
  volume={12},
  pages={4751},
}

@article{Wang2015AccurateAR,
  title={{Accurate and Reliable Prediction of Relative Ligand Binding Potency in Prospective Drug Discovery by Way of a Modern Free-Energy Calculation Protocol and Force Field}},
  author={Lingle Wang and Yujie Wu and Yuqing Deng and Byungchan Kim and Levi Pierce and Goran Krilov and Dmitry Lupyan and Shaughnessy Robinson and Markus K. Dahlgren and Jeremy R. Greenwood and Donna Lee Romero and Craig E. Masse and Jennifer L. Knight and Thomas Steinbrecher and Thijs Beuming and Wolfgang Damm and Edward D Harder and Woody Sherman and Mark L. Brewer and Ron Wester and Mark A. Murcko and Leah L. Frye and Ramy Farid and Teng Lin and David L. Mobley and William L. Jorgensen and Bruce J. Berne and Richard A. Friesner and Robert Abel},
  journal={J. Am. Chem. Soc.},
  year={2015},
  volume={137},
  pages={2695-2703},
}

@article{Ross2023TheMA,
  title={{The Maximal and Current Accuracy of Rigorous Protein-Ligand Binding Free Energy Calculations}},
  author={Gregory A. Ross and Chao Lu and Guido Scarabelli and Steven K. Albanese and Evelyne Houang and Robert Abel and Edward D Harder and Lingle Wang},
  journal={Commun. Chem.},
  year={2023},
  volume={6},
  pages={222},
}

@article{doi:10.1021/acs.jcim.0c00900,
  author = {Christina E. M. Schindler and Hannah Baumann and Andreas Blum and Dietrich Böse and Hans-Peter Buchstaller and Lars Burgdorf and Daniel Cappel and Eugene Chekler and Paul Czodrowski and Dieter Dorsch and Merveille K. I. Eguida and Bruce Follows and Thomas Fuchß and Ulrich Grädler and Jakub Gunera and Theresa Johnson and Lebrun Catherine Jorand and Srinivasa Karra and Markus Klein and Tim Knehans and Lisa Koetzner and Mireille Krier and Matthias Leiendecker and Birgitta Leuthner and Liwei Li and Igor Mochalkin and Djordje Musil and Constantin Neagu and Friedrich Rippmann and Kai Schiemann and Robert Schulz and Thomas Steinbrecher and Eva-Maria Tanzer and Andrea Unzue Lopez and Follis Ariele Viacava and Ansgar Wegener and Daniel Kuhn},
  title = {{Large-Scale Assessment of Binding Free Energy Calculations in Active Drug Discovery Projects}},
  journal = {J. Chem. Inf. Model.},
  volume = {60},
  number = {11},
  pages = {5457-5474},
  year = {2020},
}

@article{hollmann2025tabpfn,
 title={Accurate predictions on small data with a tabular foundation model},
 author={Hollmann, Noah and M{\"u}ller, Samuel and Purucker, Lennart and
         Krishnakumar, Arjun and K{\"o}rfer, Max and Hoo, Shi Bin and
         Schirrmeister, Robin Tibor and Hutter, Frank},
 journal={Nature},
 year={2025},
 month={01},
 day={09},
 doi={10.1038/s41586-024-08328-6},
 publisher={Springer Nature},
 url={https://www.nature.com/articles/s41586-024-08328-6},
}

@article{griffiths2024gauche,
  title={{GAUCHE}: A library for {Gaussian} processes in chemistry},
  author={Griffiths, Ryan-Rhys and Klarner, Leo and Moss, Henry and Ravuri, Aditya and Truong, Sang and Du, Yuanqi and Stanton, Samuel and Tom, Gary and Rankovic, Bojana and Jamasb, Arian and others},
  journal={Advances in Neural Information Processing Systems},
  volume={36},
  year={2024}
}

@article{stein_property-unmatched_2021,
  title = {Property-{Unmatched} {Decoys} in {Docking} {Benchmarks}},
  volume = {61},
  copyright = {https://doi.org/10.15223/policy-029},
  issn = {1549-9596, 1549-960X},
  url = {https://pubs.acs.org/doi/10.1021/acs.jcim.0c00598},
  doi = {10.1021/acs.jcim.0c00598},
  language = {en},
  number = {2},
  urldate = {2026-01-18},
  journal = {Journal of Chemical Information and Modeling},
  author = {Stein, Reed M. and Yang, Ying and Balius, Trent E. and O’Meara, Matt J. and Lyu, Jiankun and Young, Jennifer and Tang, Khanh and Shoichet, Brian K. and Irwin, John J.},
  month = feb,
  year = {2021},
  pages = {699--714},
}

@article{gorishniy_tabm_2025,
  title = {{TABM}: {Advancing} {Tabular} {Deep} {Learning} {With} {Parameter}-{Efficient} {Ensembling}},
  language = {en},
  author = {Gorishniy, Yury and Kotelnikov, Akim and Babenko, Artem},
  year = {2025},
}

@article{Davis2011ComprehensiveAO,
  title={Comprehensive analysis of kinase inhibitor selectivity},
  author={Mindy I. Davis and Jeremy P Hunt and Sanna Herrg{\aa}rd and Pietro Ciceri and Lisa M. Wodicka and Gabriel Pallares and Michael Hocker and Daniel K. Treiber and Patrick P. Zarrinkar},
  journal={Nature Biotechnology},
  year={2011},
  volume={29},
  pages={1046-1051},
  url={https://api.semanticscholar.org/CorpusID:32070305}
}

@article{Gilson2015BindingDBI2,
  title={BindingDB in 2015: A public database for medicinal chemistry, computational chemistry and systems pharmacology},
  author={Michael K. Gilson and Tiqing Liu and Michael Baitaluk and George Nicola and Linda Hwang and Jenny Chong},
  journal={Nucleic Acids Research},
  year={2015},
  volume={44},
  pages={D1045 - D1053},
  url={https://api.semanticscholar.org/CorpusID:8843610}
}

@article{rogers_extended-connectivity_2010,
	title = {Extended-{Connectivity} {Fingerprints}},
	volume = {50},
	issn = {1549-9596},
	url = {https://doi.org/10.1021/ci100050t},
	doi = {10.1021/ci100050t},
	number = {5},
	journal = {Journal of Chemical Information and Modeling},
	author = {Rogers, David and Hahn, Mathew},
	month = may,
	year = {2010},
	note = {Publisher: American Chemical Society},
	pages = {742--754},
}

@article{durant_reoptimization_2002,
	title = {Reoptimization of {MDL} {Keys} for {Use} in {Drug} {Discovery}},
	volume = {42},
	issn = {0095-2338},
	url = {https://doi.org/10.1021/ci010132r},
	doi = {10.1021/ci010132r},
	number = {6},
	journal = {Journal of Chemical Information and Computer Sciences},
	author = {Durant, Joseph L. and Leland, Burton A. and Henry, Douglas R. and Nourse, James G.},
	month = nov,
	year = {2002},
	note = {Publisher: American Chemical Society},
	pages = {1273--1280},
}

@article{kangas_efficient_2014,
	title = {Efficient discovery of responses of proteins to compounds using active learning},
	volume = {15},
	issn = {1471-2105},
	url = {https://bmcbioinformatics.biomedcentral.com/articles/10.1186/1471-2105-15-143},
	doi = {10.1186/1471-2105-15-143},
	language = {en},
	number = {1},
	urldate = {2025-07-25},
	journal = {BMC Bioinformatics},
	author = {Kangas, Joshua D and Naik, Armaghan W and Murphy, Robert F},
	month = dec,
	year = {2014},
	note = {Publisher: Springer Science and Business Media LLC},
}

@article{passaro2025boltz2,
  author = {Passaro, Saro and Corso, Gabriele and Wohlwend, Jeremy and Reveiz, Mateo and Thaler, Stephan and Somnath, Vignesh Ram and Getz, Noah and Portnoi, Tally and Roy, Julien and Stark, Hannes and Kwabi-Addo, David and Beaini, Dominique and Jaakkola, Tommi and Barzilay, Regina},
  title = {Boltz-2: Towards Accurate and Efficient Binding Affinity Prediction},
  year = {2025},
  doi = {10.1101/2025.06.14.659707},
  journal = {bioRxiv}
}

@article{rives2019biological,
  author={Rives, Alexander and Meier, Joshua and Sercu, Tom and Goyal, Siddharth and Lin, Zeming and Liu, Jason and Guo, Demi and Ott, Myle and Zitnick, C. Lawrence and Ma, Jerry and Fergus, Rob},
  title={Biological Structure and Function Emerge from Scaling Unsupervised Learning to 250 Million Protein Sequences},
  year={2019},
  doi={10.1101/622803},
  url={https://www.biorxiv.org/content/10.1101/622803v4},
  journal={PNAS}
}

@article{praski2025benchmarking,
  title={Benchmarking pretrained molecular embedding models for molecular representation learning},
  author={Praski, Mateusz and Adamczyk, Jakub and Czech, Wojciech},
  journal={arXiv preprint arXiv:2508.06199},
  year={2025}
}

@article{mysinger2012directory,
  title={Directory of useful decoys, enhanced ({DUD-E}): better ligands and decoys for better benchmarking},
  author={Mysinger, Michael M and Carchia, Michael and Irwin, John J and Shoichet, Brian K},
  journal={Journal of medicinal chemistry},
  volume={55},
  number={14},
  pages={6582--6594},
  year={2012},
  publisher={ACS Publications}
}

@article{tran2020lit,
  title={{LIT-PCBA}: an unbiased data set for machine learning and virtual screening},
  author={Tran-Nguyen, Viet-Khoa and Jacquemard, C{\'e}lien and Rognan, Didier},
  journal={Journal of chemical information and modeling},
  volume={60},
  number={9},
  pages={4263--4273},
  year={2020},
  publisher={ACS Publications}
}

@article{chithrananda2020chemberta,
  title={{ChemBERTa}: large-scale self-supervised pretraining for molecular property prediction},
  author={Chithrananda, Seyone and Grand, Gabriel and Ramsundar, Bharath},
  journal={arXiv preprint arXiv:2010.09885},
  year={2020}
}

@article{paricharak_analysis_2016,
	title = {Analysis of {Iterative} {Screening} with {Stepwise} {Compound} {Selection} {Based} on {Novartis} {In}-house {HTS} {Data}},
	volume = {11},
	issn = {1554-8929, 1554-8937},
	url = {https://pubs.acs.org/doi/10.1021/acschembio.6b00029},
	doi = {10.1021/acschembio.6b00029},
	language = {en},
	number = {5},
	urldate = {2026-05-06},
	journal = {ACS Chemical Biology},
	author = {Paricharak, Shardul and IJzerman, Adriaan P. and Bender, Andreas and Nigsch, Florian},
	month = may,
	year = {2016},
	pages = {1255--1264},
}

%%%%%%%%%%%%%%%%%%%%%%%%%%%%%%%%%%%%%%%%%%%%%%%%%%%%%%%%%%%%%%%%%%%%%%%%%%%%%%%
%%%%%%%%%%%%%%%%%%%%%%%%%%%%%%%%%%%%%%%%%%%%%%%%%%%%%%%%%%%%%%%%%%%%%%%%%%%%%%%
% APPENDIX
%%%%%%%%%%%%%%%%%%%%%%%%%%%%%%%%%%%%%%%%%%%%%%%%%%%%%%%%%%%%%%%%%%%%%%%%%%%%%%%
%%%%%%%%%%%%%%%%%%%%%%%%%%%%%%%%%%%%%%%%%%%%%%%%%%%%%%%%%%%%%%%%%%%%%%%%%%%%%%%
\newpage
\appendix
% \section{Failure to reach target PIC}
% \label{sec:app:failureToTargetPIC}

% \begin{table*}[htbp]
% \centering
% {\small
% \begin{tabular}{c|ccccc|ccccc}
% \toprule
%  & \multicolumn{5}{c|}{{\bf Average (Top-10)}} & \multicolumn{5}{c}{{\bf Min (Top-3)}}\\\cline{2-11}
%  & \multicolumn{5}{c|}{\textbf{Target PIC}}  & \multicolumn{5}{c}{\textbf{Target PIC}}\\
% {\bf Method} & 7.0 & 7.5 & 8.0 & 8.5 & 9.0 & 7.0 & 7.5 & 8.0 & 8.5 & 9.0 \\
% \midrule
% XGB 	& 40\% 	& -17\% 	& -8\% 	& 2\% 	& -3\% 	& 100\% 	& -18\% 	& 16\% 	& 8\% 	& 0\% \\
% MLP 	& 40\% 	& 19\% 	& -1\% 	& 0\% 	& 9\% 	& 100\% 	& 2\% 	& 12\% 	& 7\% 	& 4\% \\
% XGBReg 	& -20\% 	& 10\% 	& 0\% 	& 22\% 	& 10\% 	& -50\% 	& 2\% 	& 15\% 	& 26\% 	& 14\% \\
% \bottomrule
% \end{tabular}
% }
% \caption{{\em Failure to reach target PIC for other methods (relative to \ourmethod):}
% We show how much more often a given method fails to reach a target PIC, as compared to \ourmethod (higher is worse).
% \txtred{FILL}
% }
% \label{tbl:failureToTargetForOthers}
% \end{table*}

\section{Proofs}
\label{sec:app:proofs}

\begin{proof}[Proof of Theorem~\ref{thm:eloss}]
The proof is similar to Theorem~1 of~\cite{chakrabarti_robust_2022}.
We have
\begin{align}
E_{\bx\sim\mathcal{N}(\bx_i, \sigma^2 I)} &\left[\ell(C(\bx), y=1)\right]\nonumber\\
=& E_{\bx\sim\mathcal{N}(\bx_i, \sigma^2 I)} \max\left(0, 1 - (c + \bw^T\bx)\right)\nonumber\\
=& E_{\bv \sim \mathcal{N}(1-(c+\bw^T\bx_i), \sigma^2\bw^T\bw)} \max\left(0, \bv\right)\nonumber\\
=& s_i\cdot \Phi\left(\frac{s_i}{\sigma\|\bw\|}\right) + \sigma\|\bw\|\cdot \phi\left(\frac{s_i}{\sigma\|\bw\|}\right),\nonumber
\end{align}
where $s_i=1-(c+\bw^T\bx_i)$.
The first equality uses the form of the loss function.
The second equality follows from a change of variables.
The third equality comes from the formula for expectations of truncated normals.
\end{proof}

\section{Experimental Details}
\label{sec:app:expdetails}

\paragraph{Hyperparameters:}
In each DMTA cycle, we demeaned the ligand embedding vectors, but did not scale them.
For the classification-based methods (XGBoost and MLP), we set the top $10\%$ of the ligands in $\mathcal{D}_{seen}$ as positive.
For XGBoost and XGBRegressor, we set the trees to have a maximum depth of $6$ and an $L_2$ regularization of $1$.
For the MLP, we used one hidden layer with $32$ units.
For all methods (including \ourmethod), the best settings were chosen by optimizing the {\em average top-10} metric for a target PIC of $8$ for $<10$ proteins.
%These settings were chosen because the gave the best results for the {\em average top-10} metric for a target PIC of $8$.
We also observed that the metrics were not very sensitive to choice of parameters.
Hence, we believe that the results reflect the intrinsic abilities of these methods.

\begin{figure*}[htbp]
    \includegraphics[width=\textwidth]{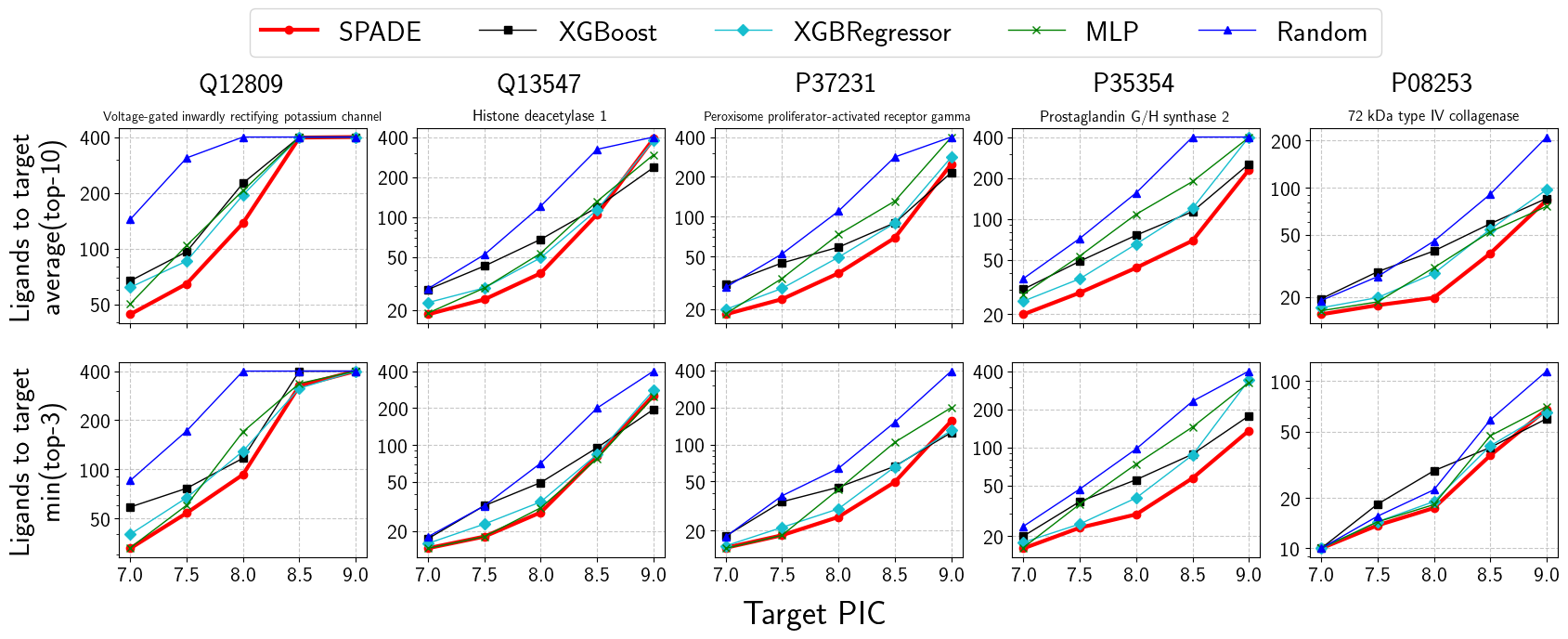}
    \caption{{\em Ligands-to-target for five example proteins (lower is better):}
    We show the number of ligand tests needed to reach a target PIC for the {\em average top-10} metric (top panel) and the {\em min top-3} metric (bottom panel).
 The UniProt IDs and names of the proteins are shown at the top.
 \ourmethod (red circles) is almost always the fastest to any target PIC.
 XGBoost (black squares) tends to be close to Random (blue triangles) initially, but improves later.
 In contrast, MLP (green crosses) is similar to \ourmethod initially, but underperforms later.
 XGBRegressor (cyan diamonds)  is in the middle.}
    \label{fig:best5}
\end{figure*}

\begin{table}
    \centering
    \footnotesize
    \begin{tabular}{c|cccc}
     & & {\bf Mean ligand tests for} & {\bf Mean ligand tests for} & {\bf Mean ligand tests for} \\
     & {\bf Median PIC} & {\bf race to 8} & {\bf race to 8.5} & {\bf race to 9}\\
    \toprule
       All ligands  & 6.0  & 33.1 & 83.4 &  273.5 \\
       Only top 1000 ligands  & 6.4 & 22.9 & 54.6 & 231.1\\
       \midrule
       {\bf Improvement} & {\bf 7.1\%} & {\bf 25.8\%}   & {\bf 31.4\%} & {\bf 18.4\%} \\
       \bottomrule
    \end{tabular}
    \vspace{1em}
    \caption{{\em \ourmethod after filtering using similar proteins:}
    We trained an XGBoost classifier to predict if a ligand has a PIC above $8$.
    Next, we applied this classifier to rank-order ligands for 10 unseen proteins.
    For each protein, we selected the top 1,000 ligands from around 16,000 candidates.
    Then, we ran \ourmethod on this filtered set of 1,000 ligands.
    We report the trimmed mean of the median PIC for both full and filtered sets.
    In addition, we also report the trimmed mean of the mean-ligands-to-target (MLT) for target PICs of 8, 8.5, and 9.
    Overall, the filtering improves the PIC distribution: the median PIC increases by 7\%.
    This filtering also enhances \ourmethod's performance for all target PICs, yielding up to a 31\% reduction in MLT at PIC 8.5.
    }
    \label{tab:filter}
\end{table}

\textbf{Compute Resources:} All experiments were run on a single workstation with an Intel Core i9-10980XE CPU (18 cores / 36 threads), 256 GB of RAM, and 2× NVIDIA RTX 3090 GPUs (24 GB each). The 100-protein benchmark is parallel across proteins and across the 50 random initialisations per protein, so we ran these as independent processes. 
%GPUs were used only for the deep-learning baselines (TabM and TabPFN); SPADE, the GP baselines (L-BFGS-B fit), XGBoost, the MLP, and the Random control all run on CPU. 
Wall-clock scoring times reported in Table~\ref{tbl:time} are per-process (not summed across parallel workers).

\end{document}